\documentclass{amsart}

\usepackage{graphicx,epsfig}

\usepackage{amsmath,amssymb,latexsym, amsfonts, amscd, amsthm, xy}

\usepackage{url}

\usepackage{hyperref}

\usepackage{mathrsfs}

\input{xy}
\xyoption{all}

\hypersetup{linkcolor=red}

\hypersetup{linkcolor=blue}

\makeindex 

=630pt \headheight = 12pt \headsep = 20pt

\hypersetup{
    colorlinks   = true,
      citecolor    = red
}

\hypersetup{linkcolor=red}

\hypersetup{linkcolor=blue}

\theoremstyle{plain}
\newtheorem{theorem}{Theorem}[section]

\newtheorem{proposition}[theorem]{Proposition}

\theoremstyle{definition}
\newtheorem{definition}[theorem]{Definition}
\newtheorem{remark}[theorem]{Remark}
\newtheorem{example}[theorem]{Example}

\def\bR{\mathbb{R}}

\def\cB{\mathcal{B}}
\def\cC{\mathcal{C}}
\def\cD{\mathcal{D}}
\def\cE{\mathcal{E}}

\def\cL{\mathcal{L}}
\def\cM{\mathcal{M}}
\def\cN{\mathcal{N}}

\def\cS{\mathcal{S}}

\def\cV{\mathcal{V}}

\def\fC{\mathfrak{C}}

\def\fb{\mathfrak{b}}
\def\fc{\mathfrak{c}}

\def\fa{\mathfrak{a}}

\def\fd{\mathfrak{d}}
\def\fg{\mathfrak{g}}
\def\fh{\mathfrak{h}}

\def\fe{\mathfrak{e}}

\def\ff{\mathfrak{f}}

\def\DIS{\mathrm{DIS}}

\def\abvg{\mathrm{AVG}_{\mathcal{E}_0(X), \mathcal{M}}}

\def\card{\mathrm{card}}

\begin{document}

\title[Local geometry of hyperedges]{A local geometry of hyperedges in hypergraphs, and its applications to social networks}

\author{Dong Quan Ngoc Nguyen and Lin Xing}

\date{September 29, 2020}

\address{Department of Applied and Computational Mathematics and Statistics \\
         University of Notre Dame \\
         Notre Dame, Indiana 46556, USA }

\email{\href{mailto:dongquan.ngoc.nguyen@nd.edu}{\tt dongquan.ngoc.nguyen@nd.edu}}

\urladdr{http://nd.edu/~dnguye15}

\email{\href{mailto:lxing@nd.edu}{\tt lxing@nd.edu}}

\maketitle

\tableofcontents

\begin{abstract}

In many real world datasets arising from social networks, there are hidden higher order relations among data points which cannot be captured using graph modeling. It is natural to use a more general notion of hypergraphs to model such social networks. In this paper, we introduce a new local geometry of hyperdges in hypegraphs which allows to capture higher order relations among data points. Furthermore based on this new geometry, we also introduce new methodology--the nearest neighbors method in hypegraphs--for analyzing datasets arising from sociology.

\end{abstract}

\section{Introduction}

One of the challenges in the modern age is to classify data arising from many resources; for example, following the rapid developments of several areas in mathematics, a large number of publications in mathematics creates a tremendous amount of data, which signifies useful information such as relationship (or collaborations) among authors and their publications, and their influences on development of mathematics. It is often the case that analyzing such data is not straightforward, and very difficult task because of the extremely fast growth of relations among data, and of data itself.

One of the main problems in machine learning is weight (respectively, label) prediction problem in which each instance is associated with a set of weights (respectively, labels), and the aim is to propose predictive models for predicting weights (respectively, labels) for unobserved data. These learning problem have important applications in many areas such as protein function classification \cite{EW}, text categorization \cite{YYTK}, and semantic scene classification \cite{BLSB}, to cite a few. For networks which can be modeled as graphs, many approaches to prediction problem were proposed (see, for example, \cite{SSSF}, \cite{KJS}, \cite{KITM}, \cite{TMF}, \cite{US},  \cite{ZZ}).  

In real world datasets, for example, in social networks, it often occurs that \textit{higher order relations} among data points are present which can not be captured because graph modeling of such datasets only signifies binary relation. So it is natural to consider weight (respectively, label) prediction problem for a generalization of graphs--hypergraphs. Recall that a hypergraph $X$ is a pair $(\cV(X), \cE(X))$, where $\cV(X)$ is the set of data points (called vertices of $X$), and $\cE(X)$ is a subset of the power set of $\cV(X)$ which represents \textit{higher order relations among data points}. Each element in $\cE(X)$ is called a \textit{hyperedge}. Since each hyperedge can consist of a collection of data points, it captures the higher order correlation information among these data points. Recent trends in learning problem have been focused on hypergraph datasets (see, for example, \cite{ABB}, \cite{ZHS})

In this paper, we propose a geometric approach for studying weight (respectively, label) prediction problem on hypergraph data which will be used to apply for analyzing networks arising from sociology. Traditionally for datasets which can be modeled as graphs, there is a standard method called \textit{$k$ nearest neighbors} which applies to weight (respectively, label) prediction problem. The classical $k$ nearest neighbors method relies on the Euclidean metrics. In contrast with many traditionally geometric approaches based on the Euclidean metrics (or distances), we propose a new methodology of nearest neighbors in hypergraphs, using \textit{metrics modulo equivalence relations} which is a weaker notion than that of usual metrics, but it seems natural to use such metrics in learning data. Equivalence relation signifies the correlation among data points, and a metric modulo such an equivalence relation signifies how much different among data points are, with respect to such correlation information. In this work, we suggest that equivalence relation is a natural mathematical notion for modeling correlation information in datasets, and that metrics modulo equivalence relations are more suitable for learning correlation in datasets. 

The structure of our paper is as follows. In Subsection \ref{subsec-metric-spaces}, we introduce a notion of metric spaces with respect to equivalence relations. In Subsection \ref{subsection-methodology}, based on the notion of metric spaces with respect to equivalence relation, we propose our methodology for studying weight (respectively, label) prediction problem. In Subsection \ref{subsection-metrics-on-hyperedges}, we construct a class of metrics modulo equivalence relations on the set of hyperedges of a hypergraph which can be combined with the methodology proposed in Subsection \ref{subsection-methodology} for studying prediction problems. In Section \ref{section-experimental-analysis}, we apply our methodology to real world datasets from sociology, and report our experimental results.


\section{A class of metrics modulo certain equivalence relations for the set of hyperedges of a hypergraph, and methodologies}

In \textit{weight (respectively, label) prediction problem}, one is given a dataset consisting of data points $x_1, \ldots, x_n$ such that there is a subset, say $\{x_{i_1}, \ldots, x_{i_k}\}$, each of whose element is associated to a weight (respectively, label). The aim is to propose a \textit{predictive model} for predicting the weight (respectively, label) of each element in the set $\{x_1, \ldots, x_n\} \setminus \{x_{i_1}, \ldots, x_{i_k}\}$. In this paper, we propose a new approach using insights from metric geometry to weight (respectively, label) prediction problem for datasets which can be modeled as \textit{hypergraphs}.

A hypergraph $X$ is a pair $(\cV(X), \cE(X))$, where $\cV(X)$ is the set of points, called \textit{vertices of $X$}, and $\cE(X)$ is a subset of the power set of $\cV(X)$, each of whose elements is called a \textit{hyperedge of $X$}. Each hyperedge $\fe$ is a subset of $\cV(X)$. If $\fe$ has exactly $m$ elements, $\fe$ is called an $m$-hyperedge of $X$. In this paper, we are interested in weight (respectively, label) prediction problem for hypergraphs, and so the following notion is natural for modeling such datasets.

\begin{definition}
(incomplete hypergraphs)
\label{d-incomplete-hypergraphs}

Let $X = (\cV(X), \cE(X))$ be a hypergraph, and let $A$ be either an interval $[a, b]$ in $\bR$ or a set of labels $\{1, \ldots, q\}$ for some positive integer $q$. $X$ is called an \textbf{incomplete hypergraph with respect to $A$} if there exists a map $F : \cE_0(X) \to A$ for some subset $\cE_0(X)$ of $\cE(X)$.

\end{definition}

\begin{remark}

In this paper, we consider two types of incomplete hypergraphs which are specified in the following.

\begin{itemize}

\item [(i)] In weight prediction problem, $A$ is an interval $[a, b] \subset \bR$, and a subset $\cE_0(X)$ of the set of hyperedges in $X$ is given such that there is a map $W : \cE_0(X) \to [a, b]$. In this case $W$ is the map $F$ in Definition \ref{d-incomplete-hypergraphs}. Each value $W(\fe)$ for $\fe \in \cE_0(X)$ is called the weight of $\fe$. Weight prediction problem asks for a \textit{predictive model} of $\overline{W}$ such that $\overline{W}$ is a map from $\cE(X)$ to $[a, b]$ for which 
\begin{align*}
\overline{W}(\fe_0) = W(\fe_0)
\end{align*}
for every $\fe_0 \in \cE_0(X)$.

\item [(ii)] In label prediction problem, $A$ is a set of labels $\{1,2, \ldots, q$ for some positive integer $q$, and a subset $\cE_0(X)$ of the set of hyperedges in $X$ is given such that there is a map $L : \cE_0(X) \to [a, b]$. In this case $L$ is the map $F$ in Definition \ref{d-incomplete-hypergraphs}. Each value $l(\fe)$ for $\fe \in \cE_0(X)$ is called the label of $\fe$. Label prediction problem asks for a \textit{predictive model} of $\overline{L}$ such that $\overline{L}$ is a map from $\cE(X)$ to $[a, b]$ for which 
\begin{align*}
\overline{L}(\fe_0) = L(\fe_0)
\end{align*}
for every $\fe_0 \in \cE_0(X)$.

\end{itemize}

\end{remark}

In classical weight (respectively, label) prediction problem for datasets which can be modeled as graphs, one can apply the classical \textit{$k$ nearest neighbors method}. The main idea of the classical $k$ nearest neighbors method is that a dataset $X = \{x_1, \ldots, x_n\}$ is equipped with a Euclidean metric (or distance), say $d$. For each $x \in X$, let $x_{i_1}, \ldots, x_{i_k}$ be the $k$ nearest points from $X$ to $x$ with respect to the metric $d$ such that
\begin{align*}
d(x, x_{i_1}) \le d(x, x_{i_2}) \le \cdots \le d(x, x_{i_k}).
\end{align*}
Then the $k$ nearest neighbors method expect that the weight (respectively, label) of $x$ should be in proximity with the $x_{i_m}$ for $1 \le m \le k$. 

Our main contribution in this paper is to introduce modified $k$ nearest neighbors methods for learning hypergraphs, especially for weight (respectively, label) prediction problem. In order to do that, we introduce a class of \textit{metrics modulo equivalence relations} on the set of hyperedges in a hypergraph. Note that in contrast with the usual metric, a metric $d$ modulo an equivalence relation $\cong$ allows that $d(\fe, \ff) = 0$ as long as $\fe \cong \ff$ (which means that $a$ is equivalent to $b$ with respect to certain properties, for example, similar weights or labels.) Thus such a metric is more suitable for studying weight (respectively, label) prediction problem since two distinct points in a dataset may have the same weight (respectively, label). In this case, the equivalence relation $\cong$ signifies the relation that certain distinct points in datasets share similar properties with respect to weight (respectively, label) map, and $d(\fe, \ff) = 0$ encodes the information that the points $\fe, \ff$ are expected to have the same weight (respectively, label). 

We first introduce a notion of metric spaces modulo equivalence relations which is used in many places in this paper.

\subsection{Metric spaces modulo equivalence relations}
\label{subsec-metric-spaces}

We present in this subsection a notion of metric spaces modulo equivalence relations. We also explain why this notion is suitable for applying to geometrical structures of hypergraphs. We first recall a notion of equivalence relations on sets.

\begin{definition}
\label{Def-equivalence-relation}
(Equivalence relation)

Let $X$ be a set. An \textit{equivalence relation}, denoted by $\cong$, on $X$ is a subset of $X \times X$ such that the following are true: 
\begin{itemize}

\item [(i)] (\textbf{Reflexivity}) $(a, a) \in \; \cong$ for every $a \in X$.

\item [(ii)] (\textbf{Symmetry}) $(a, b) \in \; \cong$ if and only if $(b, a) \in \;\cong$.

\item [(iii)] (\textbf{Transitivity}) if $(a, b) \in \; \cong$ and $(b, c) \in \;\cong$ then $(a, c) \in \; \cong$. 

\end{itemize} 

When $(a, b) \in \; \cong$, we say that $a$ is $\cong$--equivalent to $b$. Throughout this paper, in order to signify this relation, we write $a \cong b$ whenever $(a, b) \in \; \cong$. 

\end{definition}

An equivalence relation $\cong$ on a set provides a way to identify \textit{similar} elements in the set. Equivalently if one can find a \textit{measurement} to measure how similar elements in a set are, then one can modify this measurement to introduce an equivalence relation on the set. 

Reflexivity in Definition \ref{Def-equivalence-relation} says that an element in $X$ should be naturally considered to be similar to itself. If an element $a$ is similar to an element $b$, then $b$ should be naturally considered to be similar to $a$, which is exactly the symmetry condition in Definition \ref{Def-equivalence-relation}. For the last condition, Transitivity, if an element $a$ is similar to an element $b$ and $b$ is similar an element $c$, it is natural to view that $a$ is similar to $c$. 

Equivalence relation is a weaker notion than that of the identity relation which is more suitable and natural to study hypergraph data from the metric-geometry viewpoint. The main aim and intuition behind our paper is that in order to study questions in machine learning on a hypergraph $X$, it is not necessary to construct a metric on the hypergraph, i.e., a mapping $d : X \times X \to \bR_{\ge 0}$ which satisfies similar conditions as the usual absolute value in $\bR$, for example, $d(x, y) = 0$ if and only if $x = y$. Instead one only needs to construct a metric up to an equivalence relation ``$\cong$'', i.e., a mapping $d : X \times X \to \bR_{\ge 0}$ which satisfies all the conditions of a metric except the identity relation replaced by a weaker condition that $d(x, y) = 0$ if and only if $x \cong y$. It turns out that such a metric exists naturally in a very general class of hypergraphs, and can be exploited to introduce several modified methologies in machine learning which apply to weight prediction problem, and multi-label classification problem. Especially in this paper, we introduce a class of metrics for the set of hyperedges of a hypergraph which will be used to provide a modifed kNN for predicting labels of hyperedges. Up to the knowledge of the authors, this paper is the first one which realizes the existence of such metrics for the \textit{set of hyperedges} in a hypergraph. In order to obtain such metrics, we introduce a notion of neighbors of hyperedges in a hypergraph which contains a topological feature of the hyperedges. It seems to the authors that this notion of neighbors of hyperedges has never been introduced before in literature. 

We recall the notion of metrics modulo equivalence relation on a set.

\begin{definition}
\label{def-metric}
(Metric modulo an equivalence relation)

Let $X$ be a set, and $\cong$ an equivalence relation on $X$. A mapping $d : X \times X \to \bR$ is said to be a metric on $X$ modulo the equivalence relation $\cong$ if the following condition are satisfied:
\begin{itemize}

\item[(i)] $d(a, b) \ge 0$ for all $a, b \in X$.

\item [(ii)] $d(a, b) = 0$ if and only if $a \cong b$.

\item [(iii)] (\textbf{Symmetry}) $d(a, b) = d(b, a)$ for all $a, b \in X$.

\item [(iv)] (\textbf{Triangle inequality}) for any $a, b ,c \in X$, 
$$d(a, b) \le d(a, c) + d(c, b).$$

\end{itemize}

\end{definition}

A metric modulo an equivalence relation $\cong$ on a set $X$ acts almost like a metric. The only difference between a metric modulo an equivalence relation and a metric on a set is that condition (ii) in Definition \ref{def-metric} is replaced by a stronger condition that $d(a, b) = 0$ if and only if $a = b$. But in studying some properties, say $(P_i)_i$, associated to a given dataset, where each $P_i$ denotes a property of interest, it is often the case that distinct elements in the dataset share exactly the same properties $(P_i)_i$. In this case, it is natural to view that the \textit{distance} between these distinct elements as zero since they are considered to be \textit{equivalent} with respect to the properties $(P_i)_i$. Note that if one uses a usual metric on this dataset, then one cannot identify similarities among distinct elements sharing the same properties. Thus it is more natural to use a metric modulo an equivalence relation on the dataset to study the structure of the dataset related to a given set of properties. From this point of view, one can apply this idea to a variety of problems such as weight prediction, or multi-label classification as shown in this paper.

\subsection{Methodology}
\label{subsection-methodology}

In this subsection, we introduce our methodology for weight (respectively, label) prediction problem for hypergraph datasets. Throughout this subsection, let $X = (\cV(X), \cE(X))$ be a hypergraph. The main contributions of our work are the constructions of several special metric structures of the set of hyperedges $\cE(X)$. The main motivation for equipping $\cE(X)$ with such metric structures is that we want to modify the classical $k$ nearest neighbors (kNN) method (which only works if the underlying space is Euclidean), and apply it to the weight (respectively, label) prediction problem for elements in $\cE(X)$. In this work, we introduce two different paths to reach to the modification of classical kNN method, the first one called \textit{modified kNN} (which directly use the metrics we construct on $\cE(X)$), and the second one called \textit{embedded kNN} (which instead of using the metrics, we only need to make use of the local geometry around each hyperedge in $\cE(X)$, which will be also introduced also in this work.) In the following, we describe in detail our modified and embedded kNN methods.

\subsubsection{Modified kNN methods}

Assume that there is a metric modulo an equivalence relation $\cong$, denoted by $d$ on $\cE(X)$.

$\star$ \textit{Weight Prediction Problem.}

In the weight prediction problem, we assume that there is a subset $\cE_0(X)$ of $\cE(X)$ such that there is a map $W : \cS_0 \to [a, b]$ for some interval $[a, b]$ in $\bR$. The value $W(\fe)$ is called the \textit{weight of $\fe$}. The aim of the weight prediction problem is to predict what possible values of $W(\fe)$ with $\fe \in \cE(X)\setminus \cE_0(X)$ are. We propose a modified kNN method for predicting weights of elements in $\cE(X) \setminus \cE_0(X)$ as follows. 

Fix a positive integer $k \ge 1$. Take an arbitrary element $\fe \in \cE(X)\setminus \cE_0(X)$. Let $\DIS(\fe)$ denote the set of all distances from $\fe$ to elements in $\cE_0(X)$, i.e., 
\begin{align*}
\DIS(\fe) = \{d(\fe, \fe_0) \; | \fe_0 \in \cE_0(X)\}.
\end{align*}

Note that the set $\DIS(\fe)$ only consists of finitely many distinct positive real numbers, say $d_1(\fe), \ldots, d_{r_{\fe}}(\fe)$ for some positive integer $r_{\fe} \ge 1$ such that 
\begin{align*}
d_1(\fe) < d_2(\fe) < \cdots < d_{r_{\fe}}(\fe).
\end{align*}
In experimental analysis, we always choose $k$ such that $k \le r_{\fe}$ for every $\fe \in \cE(X) \setminus \cE_0(X)$. 

We only choose the $k$ smallest values in $\DIS(\fe)$, say $d_1(\fe) < d_2(\fe) < \cdots < d_k(\fe)$, and let $\text{kNN}(\fe)$ denote the set of all elements $\fe_0 \in \cE_0(X)$ such that there exists an integer $\ell_{\fe_0}$ with $1 \le \ell_{\fe_0} \le k$ for which $d_{\ell_{\fe_0}} = d(\fe, \fe_0)$, i.e.,
\begin{align}
\label{kNN-set-in-methodology-section}
\text{kNN}(\fe) = \{\fe_0 \in \cE_0(X) \; | \text{$d(\fe, \fe_0) = d_{\ell_{\fe_0}}$ for some $1 \le \ell_{\fe_0} \le k$}\}
\end{align}

We propose a \textit{predictive model} for $W$, denoted by $\overline{W}$, as follows.
\begin{align*}
\overline{W}(\fe) = \dfrac{\sum_{\fe_0 \in \text{kNN}(\fe)} W(\fe_0)}{\card(\text{kNN}(\fe))},
\end{align*}
where $\card(\text{kNN}(\fe))$ denotes the number of elements in $\text{kNN}(\fe)$. 

In experimental analysis, we perform the modified kNN method introduced above with $d$ replaced by classes of metrics modulo certain equivalence relations which we construct on the set of hyperedges $\cE(X)$ in Subsection \ref{subsection-metrics-on-hyperedges}. In Table \ref{list-of-metrics-in-modified-kNN}, we list the class of metrics we use in the modified kNN methods.

\begin{table} 
\caption{Metrics modulo certain equivalence relation used in modified kNN methods}
\centering{}%
\begin{tabular}{|c|c|}
\hline 
  \multicolumn{1}{|c|}{The metric $d$ } \tabularnewline
 
\hline 
  $\cD_{h, \cB, +\infty }$ defined in (\ref{def-3rd-metric})  \\
  \hline
  
  $\cD_{h, \cB, \epsilon}$ defined in (\ref{def-4th-metric}) \tabularnewline
\hline
\end{tabular}
\label{list-of-metrics-in-modified-kNN}
\end{table}


$\star$ \textit{Label Prediction Problem}

In the label prediction problem, we assume that there is a subset, say $\cE_0(X)$ of $\cE(X)$ such that there is a map $L : \cE_0(X) \to \{1, \ldots, q\}$ for some integer $q \ge 2$. The value $L(x_0)$ is called the \textit{label of $x_0$}. The aim of the label prediction problem is to predict what possible values of $L(x)$ with $x \in \cE(X) \setminus \cE_0(X)$ are. We propose a modified kNN method for predicting labels of elements in $\cE(X) \setminus \cE_0(X)$ as follows. 

For each $\fe \in \cE(X) \setminus \cE_0(X)$, we define the set $\text{kNN}(\fe)$ as in (\ref{kNN-set-in-methodology-section}).

We propose a \textit{predictive model} for $L$, denoted by $\overline{L}$, as follows. Let $\cL(\text{kNN})(\fe)$ be the multi-set of labels $L(\fe_0)$ for $\fe_0 \in \text{kNN}(\fe)$. If there is a \textit{mode}, say $L(\fe_0)$, in $\cL(\text{kNN})(\fe)$ for some $\fe_0 \in \text{kNN}(\fe)$, we set $\overline{L}(\fe) = L(\fe_0)$; otherwise let $\overline{L}(\fe)$ be the \textit{closest integer} to the average $\dfrac{\sum_{\fe_0 \in \text{kNN}(\fe)}L(\fe_0)}{\card(\text{kNN}(\fe))}$.



 

\subsubsection{Embedded kNN methods}

In this subsection, we introduce our embedded kNN methods. In the construction of each metric in this work, we need to introduce a \textit{local geometry at each element} in the set $\cE(X)$. The local geometry at each element needs to contain the \textit{local structure} around each element which we want to study from the set of hyperedges $\cE(X)$. Thus the local geometry at a hyperedge gathers the local hypergraph structure around the hyperedge with respect to certain properties of hypergraphs that we want to investigate. In this work, we introduce the viewpoint that the local geometry at a hyperedge $\fe$ is a set of certain hyperedges in $\cE(X)$ which we consider as sharing \textit{similar features of $\fe$ with respect to certain properties}. We introduce two different types of local geometry in the set of hyperedges which we call \textit{type I and type II neighborhoods of hyperedges}, respectively.

Table \ref{list-of-neighborhoods-in-embedded-kNN} lists types of neighborhoods for hyperedges introduced in this paper.

\begin{table} 
\caption{Types of neighborhoods used in embedded kNN methods}
\centering{}%
\begin{tabular}{|c|}
 \hline
\multicolumn{1}{|c|}{Types of neighborhoods } \tabularnewline
\hline

Type I neighborhoods $\cN_{h, \cB, +\infty }$, defined in (\ref{def-Type-I-neighborhoods-3rd-metric}) \\
 \hline
 
 Type II neighborhoods $\cN_{h, \cB, \epsilon}$, defined in (\ref{def-Type-II-neighborhoods-3rd-metric}) \tabularnewline
\hline
\end{tabular}
\label{list-of-neighborhoods-in-embedded-kNN}
\end{table}


We first consider label prediction problem.

$\star$ \textit{Label Prediction Problem}

Let $\cE_0(X)$ of $\cE(X)$ for which each hyperedge in $\cE_0(X)$ is equipped with a label in $\{1, \ldots, q\}$, i.e., there is a map $L : \cE_0(X) \to \{1, \ldots, q\}$.

In this case, let $\cN_{\text{edge}}$ denote either the type I neighborhood $\cN_{h, \cB, +\infty }$ or type II neighborhood $\cN_{h, \cB, \epsilon}$ in Table \ref{list-of-neighborhoods-in-embedded-kNN}. We define the map $T_{\cE(X)} : \cE(X) \to \bR^q$ which represents each hyperedge as a point in $\bR^q$ as follows. For each $\fe \in \cE(X)$, 
\begin{align*}
T_{\cE(X)}(\fe) = (\eta_i(\cN_{\text{edge}}(\fe)), \ldots, \eta_q(\cN_{\text{edge}}(\fe))),
\end{align*}
where for each $1 \le i \le q$, $\eta_i(\cN_{\text{edge}}(\fe))$ denotes the number of hyperedges $\ff \in \cN_{\text{edge}}(\fe)$ such that $L(\ff) = i$. We call $T_{\cE(X)}$ the \textit{feature map of $\cE(X)$}. Note that the map $T_{\cE(X)}$ is well-defined since by our notion of neighborhoods of hyperedges, each hyperedge in the neighborhood $\cN_{\text{edge}}(\fe)$ belongs to the set $\cE_0(X)$, and thus one can associate a label to $\ff$, say $L(\ff)$. 

Under the map $T_{\cE(X)}$, the set of hyperedges $\cE(X)$ can be represented as a subset, say $T_{\cE(X)}(\cE(X))$, of $\bR^q$. For label prediction problem, we view the label of each element $T_{\cE(X)}(\fe) \in \bR^q$ for a given hyperedge $\fe$ is the same as that of $\fe \in \cE(X)$. In order to predict the label of a hyperedge $\fe$, we perform the classical kNN method for the set $T_{\cE(X)}(\cE(X)) \subseteq \bR^q$ to predict the label of $T_{\cE(X)}(\fe)$ which we view as the label of the hyperedge $\fe$.

$\star$ \textit{Weight Prediction Problem}

In this case, let $\cE_0(X)$ of $\cE(X)$ for which each hyperedge in $\cE_0(X)$ is equipped with a weight in an interval $[a, b]$ in $\bR$, i.e., there is a map $W : \cE_0(X) \to [a, b]$. For an integer $q \ge 2$, we associate a map $L : \cE_0(X) \to \{1, \ldots, q\}$ to $L$ as follows. We first divide the interval $[a, b]$ into $q$ sub-intervals of equal length such that $[a_0, a_1]$, $(a_1, a_2]$, \ldots, $(a_{q - 1}, a_q]$ where $a_0 = a$, $a_q = b$, and $a_i = a_{i - 1} + (b - a)/q$ for each $1 \le i \le q$. For each $\fe \in \cE_0(X)$, we let $L(\fe) = 1$ if $W(\fe) \in [a_0, a_1]$, and $L(\fe) = i$ if $i$ is the unique integer in $\{2, \ldots, q\}$ such that $W(\fe) \in (a_{i - 1}, a_i]$. Thus one obtains the map $L : \cE_0(X) \to \{1, \ldots, q\}$. 

Repeating the same arguments as in Label Prediction Problem, one obtains the feature map defined by 
\begin{align*}
T_{\cE(X)}(\fe) = (\eta_i(\cN_{\text{edge}}(\fe)), \ldots, \eta_q(\cN_{\text{edge}}(\fe))),
\end{align*}

Under the map $T_{\cE(X)}$, the set of hyperedges $\cE(X)$ can be represented as a subset, say $T_{\cE(X)}(\cE(X))$, of $\bR^q$. For weight prediction problem, we view the weight of each element $T_{\cE(X)}(\fe) \in \bR^q$ for a given hyperedge $\fe$ is the same as that of $\fe \in \cE(X)$. In order to predict the weight of a hyperedge $\fe$, we perform the classical kNN method for the set $T_{\cE(X)}(\cE(X)) \subseteq \bR^q$ to predict the weight of $T_{\cE(X)}(\fe)$ which we view as the weight of the hyperedge $\fe$.

In our experimental analysis, for simplicity, we always let $q = 2$.

\subsection{A class of metrics for the set of hyperedges of a hypergraph}
\label{subsection-metrics-on-hyperedges}

In this subsection, we introduce a class of metrics modulo certain equivalent relations on the set of hyperedges in a hypergraph. We consider incomplete hypergraphs introduced in Definition \ref{d-incomplete-hypergraphs} which is suitable for studying the weight (respectively, label) prediction problem for the set of hyperedges.

For the rest of this section, we fix an incomplete hypergraph $X = (\cV(X), \cE(X)$ such that there is a subset $\cE_0(X)$ of $\cE(X)$ and a map $F : \cE_0(X) \to A$, where $A$ is either an interval in $[a, b]$ in $\bR$ or a set of labels $\{1, \ldots, q\}$ for some integer $q \ge 2$. In weight prediction problem, $F$ is the weight map $W$, and $A$ is an interval $[a, b]$ whereas in label prediction problem, $F$ is the label map $L$, and $A = \{1, \ldots, q\}$.

We first introduce a notion of neighborhoods of hyperedges in $X$. For each $\fe \in \cE(X)$, denote by $s(\fe)$ the size of $\fe$, i.e., the number of vertices in $\fe$. Let $\cS$ denote the set consisting of the sizes of all hyperedges in $X$, i.e., $\cS = \{s(\fe) \; | \; \fe \in \cE(X)\}$. Let $\cM : \cS \to \bR_{> 0} \cup \{+\infty\}$ be a mapping which will be used to control the sizes of neighbors of hyperedges.

\begin{definition}
($(\cE_0(X), \cM)$-neighborhoods)

Let $\fe$ be a hyperedge in $\cE(X)$. The $(\cE_0(X), \cM)$-neighborhood of $\fe$ is defined by
\begin{align*}
\cN_{\cE_0(X), \cM}(\fe) = \{\ff \in \cE_0(X) \; | \; \text{$\card(\fe \cap \ff) \ge \lfloor s(\fe)/2\rfloor$ and $s(\ff) \le M(s(\fe))$}\},
\end{align*}
where $\fe \cap \ff$ is the set of vertices that $\fe$ and $\ff$ have in common, and $\lfloor \cdot \rfloor$ denotes the floor function.

\end{definition}

For the construction of a metric for an incomplete hypergraph, for each hyperedge $\fe$, we only focus on $(\cE_0(X), \cM)$-neighbors $\ff$ of $\fe$ such that the differences between the values $\sigma(\ff)$ and the predicted value of $F(\fe)$ are very small, and can be controlled using a \textit{tuning parameter} $h$. The smaller these differences become, the more precise the predicted value of $F(\fe)$ is.

We now introduce a notion of neighborhoods of hyperedges, depending on a tuning parameter $h$.

\begin{definition}

Let $h > 0$ be a tuning parameter. For each vertex $\fe \in \cE(X)$, set 
$$\abvg(\fe) = \dfrac{\sum_{\ff \in \cN_{\cE_0(X), \cM}(\fe)} F(\ff)}{\card(\cN_{\cE_0(X), \cM}(\fe))}.$$

Set
\begin{align*}
\cN_{h, \cE_0(X), \cM}(\fe) = \{\ff \in \cN_{\cE_0(X), \cM}(\fe) \; | \; |F(\ff) - \abvg(\fe)| \le h\}.
\end{align*}

\end{definition}

\begin{figure}
\includegraphics[width=9cm,height=5.5cm]{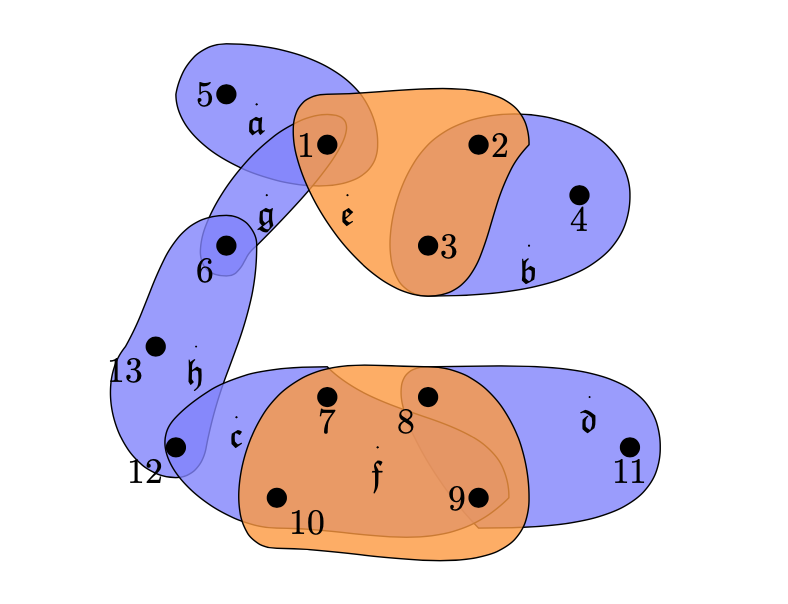}
\caption{Example of hypergraph with equivalent hyperedges.}
\centering
\label{Equal-hyperedge}
\end{figure}

For each $\fe \in \cE(X)$, set
\begin{align*}
\fC_{h, \cE_0(X), \cM}(\fe) = \card(\cN_{h, \cE_0(X), \cM}(\fe)),
\end{align*}
that is $\fC_{h, \cE_0(X), \cM}(\fe)$ is the number of elements in $\cN_{h, \cE_0(X), \cM}(\fe)$.

We introduce an equivalence relation on the set of hyperedges $\cE(X)$ which allows to identify certain hyperedges in $\cE(X)$. Note that if two hyperedges, say $\fe$ and $\ff$ satisfy $\cC_{h, \cE_0(X), \cM}(\fe) = \cC_{h, \cE_0(X), \cM}(\ff)$, then it is natural to view $\fe$ and $\ff$ as \textit{similar hyperedge} in weight (respectively, label) prediction problem since their neighborhood structures around the average mean of weights (respectively, labels) are very similar.

Hence it is natural to define a binary relation on $\cE(X)$ as follows:  \textit{two hyperedges $\fe$ and $\ff$  are \textit{equivalent}, denoted by $\fe \cong \ff$ if $\cC_{h, \cE_0(X), \cM}(\fe) = \cC_{h, \cE_0(X), \cM}(\ff)$}. One obtains the following.

\begin{proposition}

The binary relation ``$\cong$'' is an equivalence relation. 

\end{proposition}

\begin{example}

Firgure \ref{Equal-hyperedge} is an example of hypergraph that indicates which hyperedges are $\cong_3$-equivalent. 

In this example, $\cE_0(X)=\{\fa, \fb, \fc,\fd, \fg, \fh\}$. The weights of $\fa, \fb, \fc,\fd, \fg, \fh$ are $0.4, 0.5, 0.35, 0.46, -0.2, -0.15$ respectively. Let $\cM(s(\fe)) = \dfrac{4}{3}s(\fe)$ for each $\fe \in \cE(X)$. Then $\cC_{h, \cE_0(X), \cM}(\fe)=\cC_{h, \cE_0(X), \cM}(\ff)=2$, which indicates $\fe \cong \ff$.

\end{example}

Now we define a mapping $\cD_{h, \cE_0(X), \cM} : \cE(X) \times \cE(X) \to \bR_{\ge 0}$ as follows. For hyperedges $\fe, \ff$, define
\begin{align}
\label{d-metric-edges}
\cD_{h, \cE_0(X), \cM}(\fe, \ff) = |\fC_{h, \cE_0(X), \cM}(\fe) - \fC_{h, \cE_0(X), \cM}(\ff)|.
\end{align}  

\begin{theorem}
\label{3rd-main-theorem}

$\cD_{h, \cE_0(X), \cM}$ is a metric modulo the equivalence relation $\cong$.

\end{theorem}

\begin{proof}

It is clear that $\cD_{h, \cE_0(X), \cM}(\fe, \ff) \ge 0$ which shows that (i) in Definition \ref{def-metric} holds.

 Assume that $\cD_{h, \cE_0(X), \cM}(\fe, \ff)$ for some $\fe, \ff \in \cE(X)$. Then it follows from (\ref{d-metric-edges}) that $fC_{h, \cE_0(X), \cM}(\fe) = \fC_{h, \cE_0(X), \cM}(\ff)$ which implies that $\fe \cong \ff$. Thus (ii) in Definition \ref{def-metric} holds. 

It is obvious that $\cD_{h, \cE_0(X), \cM}(\fe, \ff) = \cD_{h, \cE_0(X), \cM}(\ff, \fe)$, and thus (iii) in Definition \ref{def-metric} holds.

Let $\fe, \ff, \fg$ be hyperedges in $\cE(X)$. We see that
\begin{align*}
\cD_{h, \cE_0(X), \cM}(\fe, \fg) &= |\fC_{h, \cE_0(X), \cM}(\fe) - \fC_{h, \cE_0(X), \cM}(\fg)| \\
&= |(\fC_{h, \cE_0(X), \cM}(\fe) - \fC_{h, \cE_0(X), \cM}(\ff)) + (\fC_{h, \cE_0(X), \cM}(\ff)  - \fC_{h, \cE_0(X), \cM}(\fg)) | \\
&\le  |\fC_{h, \cE_0(X), \cM}(\fe) - \fC_{h, \cE_0(X), \cM}(\ff)| + |\fC_{h, \cE_0(X), \cM}(\ff) - \fC_{h, \cE_0(X), \cM}(\fg)| \\
&= \cD_{h, \cE_0(X), \cM}(\fe, \ff) + \cD_{h, \cE_0(X), \cM}(\ff, \fg).
\end{align*}
Therefore (iv) in Definition \ref{def-metric} holds, and thus $\cD_{h, \cE_0(X), \cM}$ is a metric modulo the equivalence relation $\cong$.

\end{proof}

Depending on the range of the values of $\cM$, one can obtain different types of neighborhoods of hyperedges, and different versions of the metric $\cD_{h, \cE_0(X), \cM}$. Two distinguished cases which we study in the paper are as follows.

\begin{definition}
\label{def-Type-I-neighborhoods-3rd-metric}
(Type I neighborhoods of hyperedges)

Let $\cM(s(\fe)) = +\infty$ for each hyperedge $\fe \in \cE(X)$. For $\fe \in \cE(X)$, following the construction of the neighborhood $\cN_{h, \cE_0(X), \cM}(\fe)$ with $\cM(s(\fe))$ replaced by $+\infty$, one obtains a new type of neighborhood of $\fe$, denoted by $\cN_{h, \cE_0(X), +\infty}(\fe)$, which we call the \textit{Type I neighborhood of $\fe$}.

\end{definition}

\begin{definition}
\label{def-3rd-metric}
(A first metric on $\cE(X)$)

Let $\cM(s(\fe)) = +\infty$ for each hyperedge $\fe \in \cE(X)$. Following the construction of the metric $\cD_{h, \cE_0(X), \cM}$ with $\cM(s(\fe))$ replaced by $+\infty$, one obtains a metric modulo the equivalence relation $\cong$ which we denote by $\cD_{h, \cE_0(X), +\infty}$.

\end{definition}

The key feature of the metric $\cD_{h, \cE_0(X), +\infty}$ is that it does not control the sizes of neighbors of hyperedges. So it contains all weight information from the neighborhoods of hyperedges in its definition without \textit{filtering which weights are possibly important for prediction}.

\begin{definition}
\label{def-Type-II-neighborhoods-3rd-metric}
(Type II neighborhoods of hyperedges)

Let $\cM(s(\fe)) = \epsilon s(\fe)$for each hyperedge $\fe \in \cE(X)$. For $\fe \in \cE(X)$, following the construction of the neighborhood $\cN_{h, \cE_0(X), \cM}(\fe)$ with $\cM(s(\fe))$ replaced by $\epsilon s(\fe)$, one obtains a new type of neighborhood of $\fe$, denoted by $\cN_{h, \cE_0(X), \epsilon}(\fe)$, which we call the \textit{Type II neighborhood of $\fe$}.

\end{definition}

\begin{definition}
\label{def-4th-metric}
(A second metric on $\cE(X)$)

Let $\cM(s(\fe)) = \epsilon s(\fe)$ for each hyperedge $\fe \in \cE(X)$ where $\epsilon$ is a constant in the interval $(0, 2]$. Following the construction of the metric $\cD_{h, \cE_0(X), \cM}$ with $\cM(s(\fe))$ replaced by $\epsilon s(\fe)y$, one obtains a metric modulo the equivalence relation $\cong$ which we denote by $\cD_{h, \cE_0(X), \epsilon}$.

\end{definition}

\section{Experimental analysis}
\label{section-experimental-analysis}

In this section, we apply the methodology  proposed in Subsection \ref{subsection-methodology} combined with the class of metrics in  Subsection \ref{subsection-metrics-on-hyperedges} to construct predictive models for weight (respectively, label) prediction problem for hypergraph datasets.

We are not aware of any hypergraph datasets containing  weights (respectively, labels) of  hyperedges. So in our experimental analysis, we use datasets of weighted bipartite networks, and transform them into hypergraphs. 

Let $G=(U, V, E)$ be a bipartite network, where $U$ and $V$ are disjoint sets of vertices, $E$ is the set of edges. For each $e \in E$, $e$ is uniquely formed by a pair of vertices $(u, v)$ where $u \in U$ and $v \in V$. \textit{A hypergraph $X=(\cV(X), \cE(X))$ induced by the bipartite graph $G$} is defined as follows. 
\begin{itemize}

\item[(i)] By exchanging the notation between $U$ and $V$, one, without loss of generality, lets $\cV(X)= U$, that is, the set of vertices in $X$ is defined to be one set of vertices in $G$.

\item[(ii)]  For each $v \in V$, let $\fe_v$ denotes the set of all vertices $u$ in $U$ such that $(u, v) \in E$. We define the set of hyperedges of $X$ as
\begin{align*}
\cE(X) = \{\fe_v \; | \; v \in V\}.
\end{align*}

In this case, each hyperedge $\fe$ in $\cE(X)$ uniquely corresponds to a vertex in $V$. And thus $\cE(X)$ is in bijection with $V$. 

\end{itemize}

In our experimental analysis, we use two real-world datasets from social networks which are Epinions network and MovieLens network. They are both bipartite rating networks where one set of vertices are users and the other set of vertices are items. The weight of an edge presents a rating score from a user to an item. In the hypergraphs induced by these two networks, the set of vertices  $\cV(X)$ is the set of users, then each hyperedge is formed by the set of all users that rate the same item. Therefore, each hyperedge corresponds to an item in the bipartite graph. An example of transforming a bipartite graph into a hypergraph is shown in Figure \ref{bi-graph} and \ref{hypergraph based on bipartite graph}. Descriptions of the two datasets are as follows.

\begin{itemize} 
\item \textbf{Epinions.} This dataset was collected by Paola Massa in a $5$-week crawl (Novemeber/December 2003) from the Epinions.com Website (see the dataset at \url{http://www.trustlet.orgdownloaded_epinions.html} )(see \cite{MSST}). In Epinions, the vertices  are users and reviews, each user rates the helpfullness of a review on a $1-5$ scale, where $1$ means totally not helpful and $5$ means totally helpful. 

\item \textbf{MovieLens.}  This dataset contains movie ratings collected from \url{https://movielens.org/} (see \cite{HK}). In this dataset, the vertices are users and movies, each edge connects a user with a movie and represents a rating score. The rating scores are integers from $0$ to $5$. 

\end{itemize}

Table \ref{Des-hypergraphs} contains descriptions of hypergraphs used in our computation. The weights of hyperedges are computed using the edge weights from bipartite network datasets and  the notion of goodness introduced in \cite{SSSF}. In both of the two hypergraphs, each hyperedge corresponds to an item in the bipartite graphs,  the weight of each hyperedge, which is computed using goodness, indicates how good an item is. In order to associate each hyperedge with a  label, we divide the range of weights, which is a closed interval from minimum value of weights to maximum value of weights,  into $q$ intervals with equal length. Then the label of a hyperedge is $i$ if its weight belongs to the $i$th interval. The set of labels are integers from $1$ to $q$. 

According to the definitions of neighborhoods, one needs to select  tuning parameters $h$ for  $\cN_{h,\cE_0(X), +\infty}(\fe)$, and $ \cN_{h,\cE_0(X), \epsilon}(\fe)$. In our computation, we set the value of  $h$ based on the neighborhood of each hyperedge. For each $\fe \in \cE$ , $h$ equals to the standard deviation of weights of all hyperedges that belong to the neighborhood of $\fe$.  The choices of $k$ for $k$NN are integers from $1$ to $20$, and the values of $\epsilon$ are set to be $\frac{5}{3}, \frac{4}{3}, 1, \frac{2}{3}$, or $\frac{1}{2}$. The performances of weight predictions are evaluated using mean of absolute error (MAE) and root mean squared error (RMSE). The performances of label predictions are evaluated using error rate which is computed by the proportion of incorrect prediction. We select the values of $k$ and $\epsilon$ that  correspond to the smallest MAE or error rate. 

Figures \ref{mae-epsilon} and \ref{rmse-epsilon} present the results of predicting weights of hyperedges using Epinions dataset. In this example, the smallest MAE and RMSE values are obtained by setting $\epsilon$ equal to $\frac{5}{3}$. In general, using the second type of neighborhood ($ \cN_{h,\cE_0(X), \epsilon}(\fe)$) leads to a better prediction results than using the first type of neighborhood ($\cN_{h,\cE_0(X), +\infty}(\fe)$).

Tables \ref{pred-weight-modified} and \ref{pred-weight-embedded} present the results of predicting weights of  hyperedges using modified $k$NN method and embedded $k$NN method. In our experimental analysis, for simplicity, we set $q=2$ for the computation of $T_{\cE(X)}(\fe)=(\eta_1(\cN_{\text{edge}}(\fe)), \ldots, \eta_q(\cN_{\text{edge}}(\fe)))$ in the embedded $k$NN method, where $\cN_{\text{edge}}(\fe)$ is either type I or type II neighborhood of $\fe$. Each cell reports a pair of numbers (mean of absolute error (MAE), root mean squared error (RMSE)).

Table \ref{pred1-label-modified} presents the error rates of predicting labels of hyperedges using modified $k$NN methods. Table \ref{pred2-label-embedded}  presents the error rates of predicting labels of hyperedges using embedded $k$NN methods.


\begin{figure}
\centering
\includegraphics[width=9cm,height=10cm]{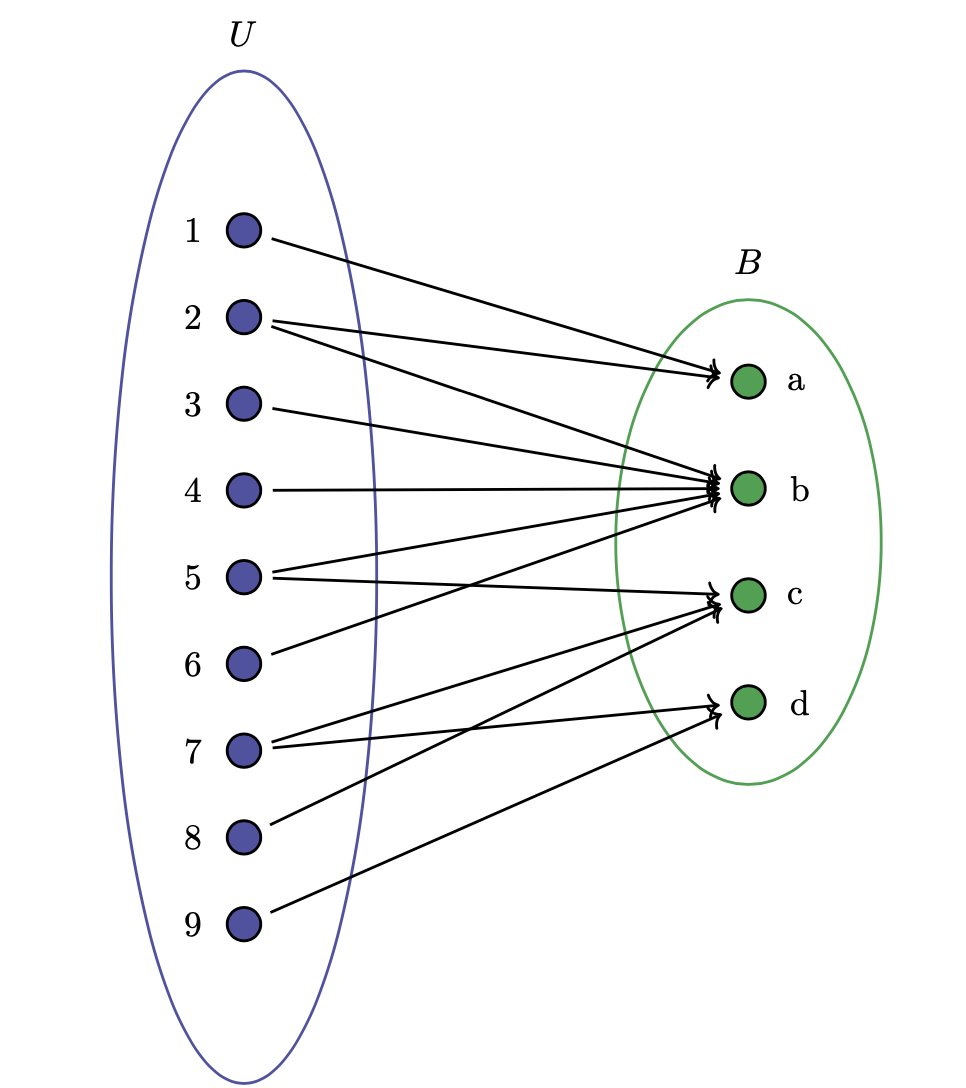}
\begin{minipage}{0.65\textwidth} 
  {\footnotesize This is a subgraph of the Epinions network. $U$ stands for the set of users from $1$ to $9$, $B$ 
stands for the set of books from $a$ to $d$.\par}
  \caption{Example of bipartite graph}
  \label{bi-graph}
\end{minipage}
\end{figure}

\begin{figure}
\includegraphics[width=9cm,height=4.5cm]{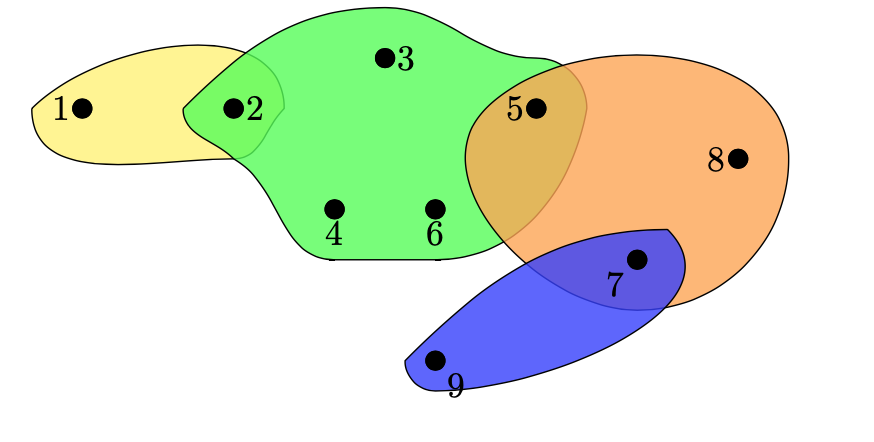}
\caption{Hypergraph based on the bipartite graph in Figure \ref{bi-graph}}
\label{hypergraph based on bipartite graph}
\centering
\end{figure}


\begin{figure}
\includegraphics[width=13cm,height=8cm]{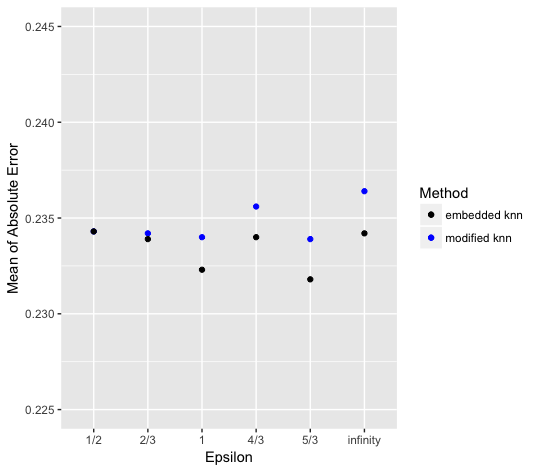}
\caption{MAE of predicting weights of hyperedges using different values of $\epsilon$.}
\centering
\label{mae-epsilon}
\end{figure}


\begin{figure}
\includegraphics[width=13cm,height=8cm]{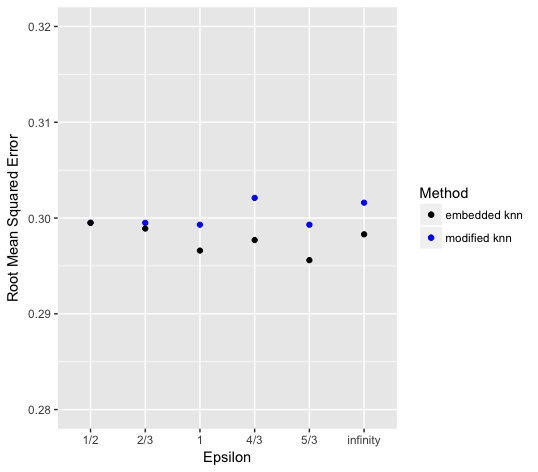}
\caption{RMSE of predicting weights of hyperedges using different values of $\epsilon$.}
\centering
\label{rmse-epsilon}
\end{figure}

\begin{table}
\caption{Descriptions of hypergraphs constructed using each dataset}
\centering{}%
\begin{tabular}{|c|c|c|c|}
\hline
& number of hyperedges & number of vertices & $s(\fe)$ :Size of hyperedge $\fe$ \tabularnewline
 \hline
 Epinions& 33774  & 32610 & $2\leq s(\fe) \leq 918$ \tabularnewline
 \hline
 MovieLens& 7498& 75106  &  $2 \leq s(\fe) \leq 687$    \tabularnewline
 \hline
 \end{tabular}
 \label{Des-hypergraphs}
\end{table}

\begin{table}
\caption{Prediction of weights of hyperedges using modified $k$NN }
\centering{}%
\begin{tabular}{|c|c|c|c|c|}
\hline
Metric&  $\cD_{h,\cE_0(X), +\infty }$ & $\cD_{h,\cE_0(X), \epsilon}$ \tabularnewline
\hline 
Epinions & (0.236, 0.301)& (0.234, 0.299) \tabularnewline
\hline 
MovieLens&(0.172, 0.217) & (0.184, 0.232)\tabularnewline
\hline 
\end{tabular}
 \label{pred-weight-modified}
\end{table}

\begin{table}
\caption{Prediction of weights  of hyperedges using embedded $k$NN }
\centering{}%
\begin{tabular}{|c|c|c|c|c|}
\hline 
 Neighborhood& $\cN_{\cE_0(X), +\infty }(\fe)$ & $\cN_{\cE_0(X), \epsilon}(\fe)$ \tabularnewline\hline 
Epinions &(0.234, 0.298) & (0.232, 0.296) \tabularnewline
\hline 
MovieLens&(0.171, 0.217) &(0.183, 0.231)\tabularnewline
\hline 
\end{tabular}
 \label{pred-weight-embedded}
\end{table}

\begin{table}
\caption{Error rates of predicting labels using modified $k$NN}
\centering{}%
\begin{tabular}{|c|c|c|c|c|}
\hline
Dataset &  \multicolumn{2}{c|}{Epinions}   & \multicolumn{2}{c|}{MovieLens}\tabularnewline
\hline
Metric&$\cD_{h,\cE_0(X), +\infty }$ & $\cD_{h,\cE_0(X), \epsilon}$ &$\cD_{h,\cE_0(X), +\infty }$ & $\cD_{h,\cE_0(X), \epsilon}$ \tabularnewline
\hline 
2 labels &0.182 (k=1) & 0.176 (k=1)& 0.250 (k=1) & 0.250 (k=4)  \tabularnewline
\hline 
3 labels & 0.316 (k=3) & 0.306(=2) & 0.296 (k=1) &0.296 (k=1)   \tabularnewline
\hline 
4 labels & 0.406 (k=4) & 0.412 (k=1) &  0.303 (k=6)& 0.303 (k=4)\tabularnewline
\hline 
5 labels & 0.507 (k=3)& 0.512 (k=4) & 0.535 (k=5)& 0.532 (k=1) \tabularnewline
\hline 

\end{tabular}
\label{pred1-label-modified}
\end{table}

\begin{table}
\caption{Error rates of predicting labels using embedded $k$NN}
\centering{}%
\begin{tabular}{|c|c|c|c|c|}
\hline
Dataset &  \multicolumn{2}{c|}{Epinions}   & \multicolumn{2}{c|}{MovieLens}\tabularnewline
\hline 
 Neighborhood&$\cN_{\cE_0(X), +\infty }(\fe)$ & $\cN_{\cE_0(X), \epsilon}(\fe)$ &$\cN_{\cE_0(X), +\infty }(\fe)$ & $\cN_{\cE_0(X), \epsilon}(\fe)$ \tabularnewline
\hline 
2 labels &0.182(k=2)&  0.176 (k=2) &  0.252 (k=19) & 0.249 (k=13)\tabularnewline
\hline 
3 labels &0.323 (k=2) &0.319 (k=5) &0.295 (k=10) & 0.296 (k=3)\tabularnewline
\hline 
4 labels &0.406 (k=10) &0.412 (k=7) & 0.304 (k=15) & 0.303 (k=10) \tabularnewline
\hline 
5 labels &0.511 (k=5) & 0.520 (k=2) &0.519 (k=15) & 0.527 (k=1) \tabularnewline
\hline 
\end{tabular}
 \label{pred2-label-embedded}
\end{table}

\end{document}